\documentclass[11pt]{article}

\usepackage[margin=1in]{geometry}

\usepackage[utf8]{inputenc}

\usepackage[T1]{fontenc}

\usepackage{microtype}

\usepackage{url}

\usepackage{graphicx}
\usepackage{float}

\usepackage{subcaption}

\graphicspath{{./Figures/}}

\usepackage{booktabs}

\usepackage{makecell}

\usepackage{amsmath}

\usepackage{amssymb}

\usepackage{mathtools}

\usepackage{amsthm}

\usepackage{dsfont}

\usepackage{nicefrac}

\usepackage{algorithm}

\usepackage{algorithmic}

\usepackage[numbers,sort&compress]{natbib}

\usepackage[colorlinks=true, linkcolor=blue, citecolor=blue, urlcolor=blue]{hyperref}

\usepackage[capitalize,noabbrev]{cleveref}

\usepackage{enumitem}

\usepackage[textsize=tiny]{todonotes}

\theoremstyle{plain}

\newtheorem{theorem}{Theorem}[section]

\newtheorem{lemma}[theorem]{Lemma}

\theoremstyle{definition}

\theoremstyle{remark}

\DeclareMathOperator*{\argmax}{arg\,max}

\DeclareMathOperator*{\argmin}{arg\,min}

\newcommand{\M}{\mathcal M}

\newcommand{\PP}{\mathcal P}

\newcommand{\LL}{\mathcal L}

\newcommand{\N}{\mathcal N}

\newcommand{\OO}{\mathcal{O}}

\newcommand{\VK}{\mathcal{V}^K}

\newcommand{\VKM}{\mathcal{V}^{K-1}}

\newcommand{\vv}{\mathbf{v}}

\def\vx{{\mathbf{x}}}

\def\vy{{\mathbf{y}}}

\DeclareMathOperator{\E}{\mathbb{E}}

\DeclareMathOperator{\R}{\mathds{R}}

\DeclareMathOperator{\AUC}{\text{AUC}}

\newcommand{\RNum}[1]{\uppercase\expandafter{\romannumeral #1\relax}}

\usepackage[margin=1in]{geometry}
\usepackage{authblk}

\usepackage[utf8]{inputenc}
\usepackage[T1]{fontenc}

\begin{document}
\begin{center}
    \hrule height 0.5pt 
    \vspace{0.2in} 
    
    {\LARGE \textbf{Graph-based Semi-Supervised Learning via Maximum Discrimination}} \\
    
    \vspace{0.2in} 
    
    \hrule height 0.5pt 
    \vspace{0.3in} 

    \textbf{Nadav Katz}\textsuperscript{1} \quad
    \textbf{Ariel Jaffe}\textsuperscript{1} \quad
    \\[1em] 
    
    \textsuperscript{1}Department of Statistics and Data Science \\
    The Hebrew University of Jerusalem
    \\[0.5em] 
    
    \texttt{\{nadav.katz3, ariel.jaffe\}@mail.huji.ac.il}
\end{center}

\begin{abstract}
 Semi-supervised learning (SSL) addresses the critical challenge of training accurate models when labeled data is scarce but unlabeled data is abundant. Graph-based SSL (GSSL) has emerged as a popular framework that captures data structure through graph representations. 
Classic graph SSL methods, such as Label Propagation and Label Spreading, aim to compute low-dimensional representations where points with the same labels are close in representation space. 
Although often effective, these methods can be suboptimal on data with complex label distributions. 
In our work, we develop AUC-spec, a graph approach that computes a low-dimensional representation that maximizes class separation. We compute this representation by optimizing the Area Under the ROC Curve (AUC) as estimated via the labeled points. We provide a detailed analysis of our approach under a product-of-manifold model, and show that the required number of labeled points for AUC-spec is polynomial in the model parameters. Empirically, we show that AUC-spec balances class separation with graph smoothness. It demonstrates competitive results on synthetic and real-world datasets while maintaining computational efficiency comparable to the field's classic and state-of-the-art methods.  
\end{abstract}

\section{Introduction}
Semi-supervised learning (SSL) algorithms are motivated by applications where labeled samples are expensive to obtain, whereas unlabeled samples are cheap and abundant.
Given a small number of labeled samples and a large number of unlabeled ones, the main challenge is to leverage both the labeled and unlabeled samples to train an accurate classifier.

In recent decades, a large body of research has focused on the development of various approaches to SSL \cite{10.7551/mitpress/6173.003.0017}. In model-based methods, one assumes that the data was generated according to some generative model, such as a Gaussian or other mixture models. Inference of the model parameters can be done in an unsupervised manner by applying algorithms such as Expectation-Maximization (EM) or gradient descent to maximize the model's likelihood function.
In cases where the statistical model provides a good approximation to the sample distribution, such an approach can be highly effective. Another recent body of work derived semi-supervised methods based on augmented samples with high-confidence labeling \cite{sohn2020fixmatch}. 

Graph-based, data-driven methods are a different class of SSL algorithms that do not rely on a predefined statistical model. Instead, the unlabeled samples are used to compute a graph whose connectivity profile captures the latent structure of the dataset. For example, by analyzing the graph spectra, one can detect clusters of observations or learn the shape of a low-dimensional manifold \cite{von2007tutorial,coifman2006diffusion}.

One of the first graph-based algorithms for SSL was developed in \cite{belkin2004semi}. The algorithm computes a low-dimensional representation for each observation based on the leading eigenvectors of the graph Laplacian matrix. 
Then, a classifier or regressor is trained to predict the label based on the new set of features. \citet{zhu2002learning} developed the Label Propagation (LP) approach, 
where one or several predictive vectors are obtained by applying power iterations with a graph random walk operator, while fixing the elements that correspond to the labeled samples to a predefined value.
\citet{zhu2003semi} derived a related optimization-based approach that computes a classifier by minimizing a quadratic cost function on the graph-Laplacian matrix, while keeping the constraints imposed by the labeled examples.  \citet{zhou2003learning} introduced \textit{label spreading}. While LP strictly clamps the labeled points to their given values, label spreading introduces a regularization parameter that allows more flexibility in the labeled points during the propagation process.

A significant limitation of LP and its optimization variants is that they may yield uninformative results when the number of unlabeled samples is very large. Specifically, the output vector may become  \textit{spiked}, matching the labels at the labeled points but nearly constant elsewhere
 \cite{nadler2009semi}.
To address this issue, several works developed variations of LP. \citet{el2016asymptotic} analyzed the asymptotic behavior of the solution to an $\ell_p$ Laplacian optimization problem and showed that, under minor assumptions, the solution is continuous. \citet{zhou2011semi} derived a higher-order graph-regularization scheme and showed that it addresses the problem raised by \cite{nadler2009semi}.
\citet{calder2020poisson} show that LP fails when the propagation time to a given point is beyond the \textit{mixing time} of the random-walk process determined by the graph. To address this, the authors developed Poisson learning that  considered only short walks along the graph. \citet{holtz2023semi} showed that computing multiple prediction vectors under an orthogonality constraint improves the outcome of LP. 

All the above methods aim to compute prediction vectors whose values are similar for points sharing the same label. As we demonstrate in our paper, in many cases, this requirement is overly restrictive and leads to suboptimal results. Here, we take a different approach that promotes \textit{separation} between labeled samples with different labels. With this approach, we make the following contributions: 
(i) We develop a graph-based SSL method termed AUC-spec that maximizes an AUC-based score function. (ii) We show theoretically, under a cluster and manifold setting, that the outcome of AUC-spec is smooth with respect to the graph and does not become spiked for a large number of samples. (iii) For a product-of-manifold model, we compute a probabilistic bound on the number of labeled points required for obtaining a non-trivial predictor. (iv) We provide empirical evidence that our method obtains results that are competitive and often outperform competing, state-of-the-art graph-based methods, with a significantly improved runtime. 

\subsection{Related work: AUC-based optimization for SSL}

The adoption of the Area Under the ROC Curve (AUC) as an optimization objective in semi-supervised learning (SSL) arises from its inherent advantage in moderate class-imbalanced scenarios, where traditional accuracy metrics fall short. AUC evaluates a classifier’s probability to rank positive samples above negative ones \cite{cortes2003auc}.
In recent years, methods for optimizing AUC have been developed in many domains (e.g., \cite{rakotomamonjy2004optimizing}, \cite{herschtal2004optimising}, \cite{wang2016auc}). For example, supervised AUC optimization methods such as AUC-LS use approximations of the AUC metric to enable gradient-based optimization \cite{fujino2016semi}.

In recent years, several works have incorporated AUC optimization in the framework of SSL. Several works combined AUC loss functions with model-based likelihood functions and estimated model parameters by optimizing over an integrated loss function \cite{fujino2016semi}. Other works developed model-free methods that utilize unlabeled data to improve the estimate of the AUC score. 
For example, inspired by transductive learning, \cite{wang2015optimizing} developed a method to compute a boundary between two classes while simultaneously ranking unlabeled points according to their location with respect to the boundary. \cite{sakai2018semi} and \cite{xie2018semi} developed a method that estimates AUC by the positive labeled samples while treating the unlabeled points as negative, and vice versa. A similar approach is taken in \cite{iwata2020semi} to compute an estimate for the partial AUC score. 
\cite{xie2024weakly} provides a unified approach to SSL with other weakly supervised settings, such as the existence of noisy labels.

In our work, we incorporate AUC methods with graph-based methods. Thus, as mentioned, we do not rely on a predetermined generative model. The goal is to obtain a low-dimensional representation of the data that reflects the underlying structure, as captured by the graph, while maintaining separation between samples with different labels via AUC scores. 
In the following sections, we provide the necessary background on graph-based methods and describe our  approach in detail.

\subsection{Problem setting and preliminaries}
\label{sec:perliminaries}


Throughout the paper we use bold notation for vectors and capital letters for matrices.
Let $\mathcal{X} = \left\{\mathbf{x_i}\right\}_{i=1}^n$ be a dataset of $n$ observations, where $\mathbf{x}_i \in \mathbb{R}^d$ is the feature vector of the $i$-th data point. Among these, only a subset $\mathcal{L} \subset \{1, ..., n\}$ is labeled, with corresponding labels $\mathbf y_{\mathcal{L}} = [\mathbf{y}_i]_{i \in \mathcal{L}}$. 
Our goal is to estimate the labels for the unlabeled samples. 

\section{Graph-based methods for SSL}

In this section, we provide the necessary background on graph-based SSL algorithms. We specifically focus on two classical methods: (i) label propagation and (ii) the leading eigenvectors approach.

 


\paragraph{Label Propagation.} Let $G = (V, E)$ be an undirected weighted graph, where each node $v_i$ corresponds to a single observation $\vx_i$. The weight between two nodes $\vv_i,\vv_j$, denoted $W_{ij}$, is computed by,
\begin{equation}\label{eq:adaptive_scaling}
W_{ij}=\exp\left(\frac{-||\mathbf x_i-\mathbf x_j||^2}{d_K(\mathbf x_i)d_K(\mathbf x_j)}\right), 
\end{equation}
where $d_K(\mathbf x_i)$ is the distance between $\mathbf x_i$ and its $K$-th nearest neighbor.
This adaptive scaling approach for edge weights helps capture the local neighborhood structure more effectively, adjusting for variations in data density in the feature space \cite{NIPS2004_40173ea4}.
Using the kernel matrix $W$, we compute the degree matrix $D$, a diagonal matrix where $D_{ii} = \sum_{j} W_{ij}$. We denote by $L$ the unnormalized Laplacian matrix given by $L= D-W$, and by $L_{rw}$ the random walk operator given by $ = D^{-1}W$.

One of the most influential methods in GSSL was derived by \citet{zhu2003semi}. The labeled points are used as constraints in the following optimization problem.
\begin{equation}\label{eq:ssl_label_prop}
    \hat {\mathbf{v}} =  \min  \left(\mathbf{v}^T L \mathbf{v}\right) \quad \text{subject to} \quad \mathbf{v}_i = \mathbf{y}_i \qquad \forall \space i\in \mathcal{L},
\end{equation}
The loss function in Eq. \eqref{eq:ssl_label_prop}  encourages smoothness over the graph while respecting the known labels. For a binary setting, the labels for the unlabeled points are computed by rounding the elements of $\vv$. \citet{zhu2002learning} derive Label Propagation (LP), which can be viewed as a numeric solution to  Eq. \eqref{eq:ssl_label_prop}.  LP combines power iterations with the random-walk operator while enforcing constraints over the labeled points at each iteration. Alg. \ref{alg:label_propagtion} provides a pseudocode for LP for the case of binary labels.
\begin{algorithm}[tb]
\caption{Label Propagation for binary labels}
\label{alg:label_propagtion}
\textbf{Input}:  $L_{rw}$ (random walk operator), $\mathbf{y}_\mathcal{L}$ (labels for subset of points), $\delta$ (convergence threshold)
\\
\textbf{Output}: Estimated labels $\hat \vy$. 
\begin{algorithmic}[1] 
\STATE Initialize $\mathbf{v}^{(0)}$, by setting elements $v_i$ at random. 
\WHILE{$\|v^{(t+1)}-\vv^{(t)}\|>\delta$}
\STATE $\mathbf{v}^{(t+1)} \gets L_{rw}\mathbf{v}^{(t)}$
\STATE $\mathbf{v}^{(t+1)}_\mathcal{L} \gets \mathbf{y}_\mathcal{L}$
\ENDWHILE
\STATE $\hat \vy \gets \mathds 1(\vv^{(t+1)} \geq 0.5)$
\end{algorithmic}
\end{algorithm}

\paragraph{Leading eigenvectors} 
This method, derived by \citet{belkin2004semi} computes the solutions to the generalized eigenvector problem,
\begin{equation*}
    L\mathbf{v} = D\lambda\mathbf{v}.
\end{equation*}
The $k$ solutions that correspond to the smallest eigenvalues are set as new features for the regression or classification task. 
A supervised learning algorithm is then trained based on the labeled points. 
The underlying assumption is that these eigenvectors capture the most important structural information of the data, and that the label function lies in the span of these eigenvectors.

\paragraph{Limitations of LP and leading Laplacian eigenvectors methods, an illustrative example.}


In LP, the optimization process aims for a solution that is smooth with respect to the graph, with values equal to the labels at the labeled points. Such a solution exists, for example, in scenarios where class labels are aligned with well-separated clusters. 
However, in other settings, such an alignment may not hold. For example, labels may be determined by clusters with partial overlap. Alternatively, a label may be defined by a continuous, latent parameter. Thus, a vector that is smooth with respect to the graph, and equals the labels at specific points may not exist. 



%
In the leading eigenvectors method developed in \cite{belkin2004semi}, we assume that a predefined number of  Laplacian eigenvectors capture the necessary information for accurate classification. This assumption, however, may not hold in many important applications. Specifically, the information required for discriminating between classes may reside deeper in the spectrum.

For illustration, consider the Ring of Gaussians dataset  created by \citet{holtz2023semi}. This dataset consists of eight Gaussian clusters, positioned equidistantly along a circular ring in a 2D plane, with a radius $R=5$. The mean of each Gaussian is given by $\mu_i = \left(R\cos \theta_i, R\sin \theta_i \right)$, where $\theta_i = \frac{2\pi i}{8}$ and the points are sampled from the distribution $\mathcal{N}(\mu_i, \sigma^2I)$, with $\sigma^2 = 0.5$. Each cluster contains an equal number of points, and the labels are assigned alternately as $1$ or $0$, reflecting a binary classification setup. 
Figure \ref{fig:ring_of_gaussians_illustration} (left) provides a visualization of this dataset with $3000$ points.

\begin{figure*}
    \centering
        \centering
        \includegraphics[width=0.9\linewidth]{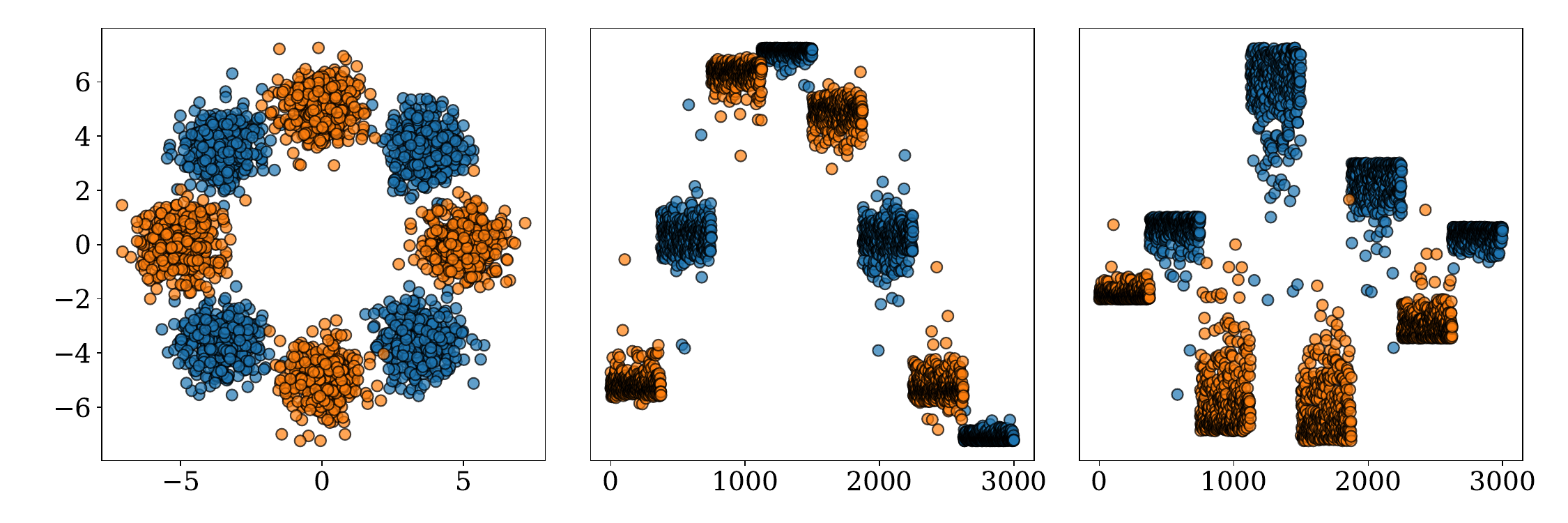} 
    \hfill
    \hfill
    \caption{Ring of Gaussians dataset, colored by class value with labeled points in hand colored by red}
    \label{fig:ring_of_gaussians_illustration}
\end{figure*}

The ring-of-Gaussians dataset illustrates several of the limitations of current approaches. Here, adjacent clusters have different labels despite the proximity of some of their points in the feature space. The leading (non-trivial) eigenvector captures the global structure of the ring, but does not align with the class boundaries. The eigenvector that best separates the classes is not among the first few vectors and happens to be the 7th in this case, see Fig. \ref{fig:ring_of_gaussians_illustration}.
In addition, even though the 7th vector discriminates between classes almost perfectly, points with similar labels do not necessarily correspond to elements with similar values. Thus, enforcing the values of a subset of labeled points is too restrictive and may yield a suboptimal solution. 


These observations highlight a critical gap in addressing real-world challenges. 
To overcome these challenges, we propose a novel graph-based approach. Rather than enforce specific values for the labeled points, our method computes a representation for each point that aims to separate labeled points with different labels. To that end, we formulate an optimization problem that leverages the Area Under the ROC Curve (AUC) as an optimization objective. Our AUC-guided spectral learning algorithm is introduced in the next section.

\section{Method}
We first describe our approach for the binary case, and then extend it to the multi-class classification settings.

\subsection{AUC-Guided Spectral Optimization for binary classification}
\label{subsec:binary_method}
Our approach is based on combining two criteria: (i) smoothness over the graph, as measured by the quadratic form over the Laplacian $L$, and (ii) The area under the ROC Curve (AUC) with respect to the labeled points. 
Let $J\left(\mathbf{v}\right)$ denote the \textit{ideal}, non-convex objective function given by
\begin{equation}\label{eq:objective}
    J\left(\mathbf{v}\right) = \mathbf{v}^TL\mathbf{v} - \gamma \text{AUC}\left(\mathbf y_\mathcal{L}, \mathbf{v}_\mathcal{L}\right).
\end{equation}
We denote by $\mathcal P$ and $\mathcal N$ the set of positive and negative labeled samples, respectively. 
The term $\text{AUC}\left(\mathbf y_\mathcal{L}, \mathbf{v}_\mathcal{L}\right)$ in the RHS of Eq. \eqref{eq:objective} is equal to 
\[
\text{AUC}\left(\mathbf y_\mathcal{L}, \mathbf{v}_\mathcal{L}\right) = \E_{i\in \mathcal{P}, j\in \mathcal{N}} [\mathds{1}(\vv_i > \vv_j)].
\]
where the expectation is taken over random pairs of nodes $i\in \mathcal{P}, j\in \mathcal{N}$, and $\mathds{1}(\cdot)$ is the indicator function.
Since the indicator function is non-differentiable, we approximate its value by using the sigmoid function $\sigma(z)=\frac{1}{1+e^{-z}}$:
$$\text{AUC}\left(\mathbf y_\mathcal{L}, \mathbf{v}_\mathcal{L}\right) \approx \mathbb{E}_{i\in \mathcal{P}, j\in \mathcal{N}}\left[\sigma\left(\mathbf{v_i}-\mathbf{v_j} \right)\right].$$

As a numeric solution to the problem in Eq. \eqref{eq:objective}, we incorporate power iterations, modified with a gradient-based update rule:
\begin{align}
    &\mathbf{v}^{(t+1)} = \mathbf{v}^{(t)} 
    \notag \\
    &+ \gamma\left(L_{rw}\mathbf{v}^{(t)} + \nabla\left(\frac{1}{|\mathcal P||\mathcal N|}\sum_{i\in \mathcal P }\sum_{j\in\mathcal N}\sigma\left(\mathbf{v}^{(t)}_i-\mathbf{v}^{(t)}_j \right)\right)\right)
\end{align}
Let $\sigma_{ij}=\sigma(\mathbf{v}_i - \mathbf{v}_j)$. The gradient for labeled nodes equals,
\begin{equation}\label{eq:auc_grad}
\big(\nabla\text{AUC}\left(\mathbf{y}_\mathcal{L}, \mathbf{v}_\mathcal{L}\right)\big)_i = 
\begin{cases}
    \frac{1}{|\mathcal P||\mathcal N|}\sum_{j\in\mathcal N}\sigma_{ij}(1-\sigma_{ij}), & i\in \mathcal P\\

    \frac{-1}{|\mathcal P||\mathcal N|}\sum_{i\in\mathcal P}\sigma_{ij}(1-\sigma_{ij}), &  i\in \mathcal N 
    \\
    0 & \text{o.w.}
\end{cases}
\end{equation}

After convergence, the labels are computed via $\hat \vy_i = \mathds 1(\vv_i \geq 0)$. 
Alg. \ref{alg:auc_power_method} presents pseudocode for our algorithm. 
\begin{algorithm}[h]
\caption{AUC-Guided Spectral Optimization (AUC-Spec)}
\label{alg:auc_power_method}
\textbf{Input:} Graph Laplacian $L_{rw}$, label vector $y_\LL$, step size $\gamma$, tolerance $\text{tol}$
\\
\textbf{Output:} Estimated labels $\hat \vy$
\begin{algorithmic}[1]
\STATE Initialize $\mathbf v^{(0)} \gets (-1, \dots, y_\mathcal{L}, \dots, -1)$
\STATE $\mathbf v^{(0)} \gets \mathbf v^{(0)} / \|\mathbf v^{(0)}\|$
\WHILE{$\|\mathbf{v}^{(t+1)} - \mathbf v^{(t)}\| > \text{tol}$}
    \STATE $(\nabla \text{AUC})_i \gets 
    \begin{cases}
    \frac{1}{|\mathcal P||\mathcal N|}\sum_{j\in\mathcal N}\sigma_{ij}(1-\sigma_{ij}) & i \in \mathcal P \\
    \frac{1}{|\mathcal P||\mathcal N|}\sum_{j\in\mathcal N}\sigma_{ij}(1-\sigma_{ij}) & i \in \mathcal N \\
    0 & i \notin \LL
    \end{cases}$
    \STATE $\mathbf{v}^{(t+1)} \gets \mathbf v^{(t)} + \gamma (L_{rw} \mathbf v^{(t)} + \nabla \text{AUC})$ \COMMENT{Update step} 
    \STATE $\mathbf{v}^{(t+1)} \gets \mathbf{v}^{(t+1)} / \|\mathbf{v}^{(t+1)}|$
\ENDWHILE
\STATE $\hat \vy_i = \mathds 1(\hat \vv_i \geq 0)$. 
\STATE \textbf{Return} $\hat \vy$. 
\end{algorithmic}
\end{algorithm}


\subsection{Method for multi-class classification}\label{subsec:multiclass_method}

Our method can be extended to multi-class problems using a One-vs-Rest strategy, where each class is treated as a binary classification task against all others. For each class $c \in \{1, \dots, C\}$, we treat its labeled nodes as positive ($\mathcal{P}_c$) and all others as negative ($\mathcal N_c$). The gradient ascent iterations are modified to optimize class-specific AUC objectives at each iteration $t$ with the following steps:
\begin{itemize}[left=0pt]
    \item For each class $c$, define $\mathcal{P}_c = \{i\in \mathcal L, \mathbf{y}_i = c\}$, $\mathcal{N}_c = \mathcal{L} \setminus \mathcal{P}_c$.
    \item Compute the vector $\vv^{(c)}$ by applying Alg. \ref{alg:auc_power_method} with input $\PP_c$ and $\N_c$.
    \item For all unlabeled nodes $i$ predict the class via $\hat \vy_i = \argmax_{c} \mathbf{v^c}_i$.
\end{itemize}

\paragraph{Computational complexity.}
The computational complexity of AUC-spec is similar to other methods with standard power iterations. These  primarily involve matrix-vector multiplication at $\mathcal{O}(Cn^2)$ per iteration, where $C$ is the number of classes. For large $n$, the cost can be reduced to $\mathcal{O}(C|E|)$ when the graph is sparse.
Our method adds an AUC-based gradient update computed over all labeled positive-negative pairs, contributing an additional cost per iteration of order $\mathcal{O}(C|\mathcal P| |\mathcal N|)$. However, since this term depends solely on the number of labeled points and not the full dataset, it remains relatively small, especially under  low label rate regimes. 
In the simulations, we show that our approach converges with fewer iterations and is advantageous in terms of runtime compared to other methods.


\section{Gaussian mixture and product-manifold models}

In this section, our aim is to gain insight into the settings where our approach can improve upon label propagation. Specifically, we intend to illustrate why enforcing label values, either directly as in LP or its softer label spreading version, may yield suboptimal results.  
To that end, we use two examples: (i) A mixture of two Gaussians, and (ii) A product-manifolds model. 

\paragraph{A Gaussian mixture model.} Fig. \ref{fig:GaussianMixture} (left) shows samples generated by the following $2D$ Gaussian Mixture model. The mean value of the Gaussian equals $(1,1)$ and $(-1,-1)$. The covariance matrix of both Gaussians is set to $1.5I$ such that there is some overlap between the two Gaussians. The samples in the figure are colored according to which Gaussian they were sampled from, with $5$ labeled points colored in red. Note that there is asymmetry in the labeled points, with two labeled points from the yellow Gaussian and three labeled points from the blue Gaussian. One of the blue-labeled points is located close to the origin. 

The middle panel shows the predictions made via LP. As can be seen, the border between the positive and negative predictions does not pass through the origin. 
The random-walk interpretation of LP provides a simple explanation for this result. Following \cite{calder2020poisson}, let $s_i^{(t)}$ define a random walk process beginning at a node $i$, whose node-to-node transition probabilities are given by the elements of $L_{rw}$. Let $\tau$ be the first time the process hits a labeled point, and let $y_i^\tau$ be the label obtained at time $\tau$. The solution to the optimization problem in Eq. \eqref{eq:ssl_label_prop} is $\E[y_i^\tau]$. 

In our example, there is significant overlap between the Gaussians, and the location of the labeled points is non-symmetric with respect to the origin. Thus, a random walk process initiated from yellow points close to the origin  will reach a blue-labeled point before reaching a yellow one, with high probability. In contrast, consider the solution to the optimization problem in Eq. \eqref{eq:objective}. The leading non-constant eigenvector of $L$ (i.e., excluding the trivial eigenvector corresponding to eigenvalue 0) separates the two Gaussians. In addition, it has a perfect AUC score as it separates the blue and yellow labeled points. Thus, it minimizes the objective in Eq. \eqref{eq:objective}. The right panel in Fig. \ref{fig:GaussianMixture} shows the prediction made by AUC-spec (Alg. \ref{alg:auc_power_method}). Here, the border between the predictions goes through the origin, as expected.


\begin{figure*}
    \centering
    \includegraphics[width=0.8\linewidth]{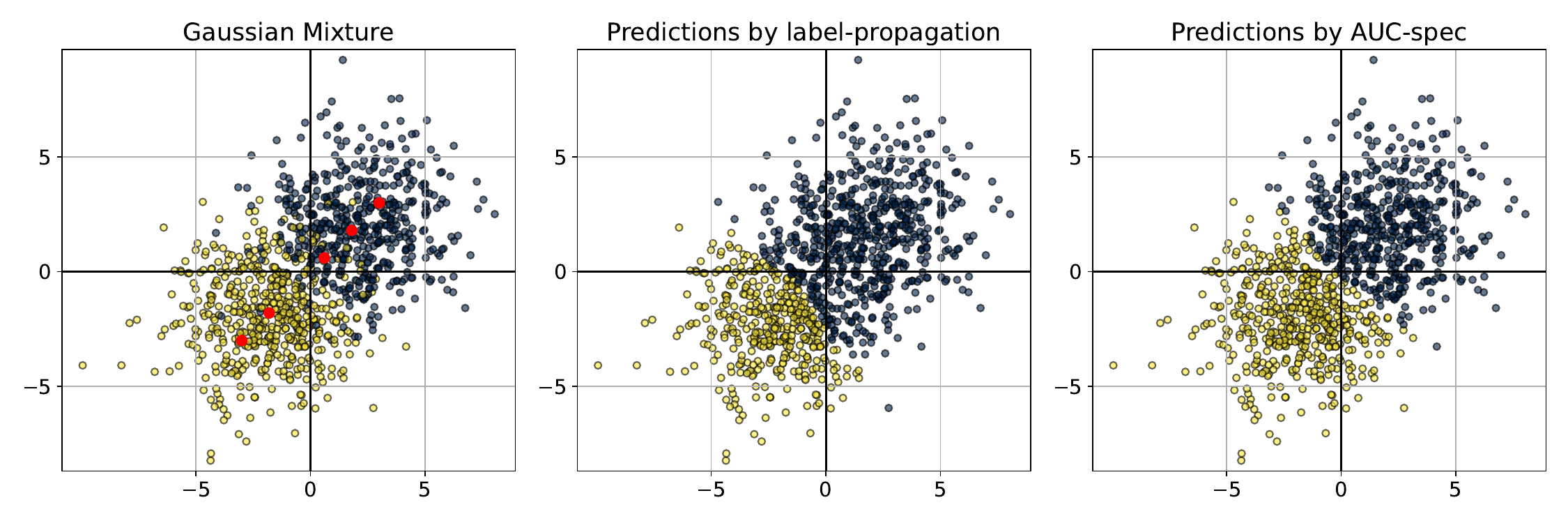}
    \caption{A Gaussian mixture model: (Left) Points are colored according to their class. Labeled points are colored in red. (Middle and right): points colored according to the outcome of label propagation (middle) and AUC-spec (right).}
    \label{fig:GaussianMixture}
\end{figure*}

\paragraph{A product-manifold model.}
The Gaussian-Mixture model is an example of a  cluster setting, where the leading (non-constant) eigenvector of $L$ perfectly separates the samples from the two classes. This implies that it is the solution $\hat \vv$ of \eqref{eq:objective}, as it minimizes both parts of the objective. 
In this example, we study the manifold setting where the label is set according to a continuous latent parameter. This case is more problematic as the leading eigenvectors may not be relevant to the class labels. To illustrate this phenomenon, we study the product-manifold setting. 

Let $\M$ be a low-dimensional manifold embedded in $p$ dimensions. In the manifold model, the high-dimensional observations $\vx_i \in \R^p$ can be approximated by a smooth, potentially nonlinear map $T$
of a small number of variables $\theta \in \R^d$ such that $\theta_i \xrightarrow{T}  \vx_i$. Let $\M^{(1)},\M^{(2)}$ be two manifolds with internal dimension $d_1$ and $d_2$, respectively, such that  $\theta^{(1)} \xrightarrow{T_1} \vx^{(1)}$, and $\theta^{(2)} \xrightarrow{T_2} \vx^{(2)}$. In a product-manifold, denoted $\M = \M^{(1)} \times \M^{(2)}$, there is a bijection between the points $\vx$ in the product, and pairs of points $\vx^{(1)} \in \M^{(1)}$ and $\vx^{(2)} \in \M^{(2)}$. A simple example of a product-manifold is obtained by concatenating the points $\vx^{(1)}$ and $\vx^{(2)}$. 

For simplicity, we assume that our data is generated at random according to a uniform distribution over a product of two manifolds, where the latent dimension of $\M^{(1)}$ is $d_1=1$, such that $\theta^{(1)}$ is a scalar.
We consider it instructive to examine the asymptotic setting where the number of samples $n \to \infty$. Here, under several assumptions regarding the kernel function in Eq. \eqref{eq:adaptive_scaling}, the Laplacian eigenvectors converge to samples of the eigenfunctions of the Laplace-Beltrami (LB) operator of the manifold. 
Due to a lack of space, we relegate a more detailed discussion of manifold products and the conditions for eigenvector convergence to the supplementary material. Here, we state two key properties of product-manifolds that are most relevant to our analysis.  
\begin{itemize}
    \item[\RNum{1}] There is an eigenvector $\vv_K$ of the graph Laplacian whose elements $(\vv_K)_i$ are monotonic with respect to $\theta^{(1)}$. 
    \item[\RNum{2}] The elements of eigenvectors $\vv_k$ for $k=1,\ldots,K-1$ are statistically independent of the elements of $\vv_K$.
\end{itemize}

We make the following additional assumptions:
\begin{itemize}[left=2pt]
    \item[(i)] The true labels of all points are set according to the value of $\theta^{(1)}$, such that $\vy_i = \mathds 1(\theta^{(1)}_i \geq 0)$.  For simplicity, we assume an equal number of positive and negative labeled points, sampled (separately) uniformly at random.  
    \item[(ii)] The elements of the leading eigenvectors of the Laplacian matrix $\vv_1,\ldots,\vv_K$ are relatively balanced, such that $(\vv_k)_i \leq  C/\sqrt{n}$, for all $i$ and vectors $k=1,\ldots,K$.
\end{itemize}

By property \RNum{1} and assumption (i), there is some threshold $\tau$ that satisfies $\vy_i = \mathds 1(\vv_K \geq \tau)$. Thus, the eigenvector $\vv_K$ is, in some sense, an ideal separator between the two classes in the data. 
Recall that $\lambda_K$ is the eigenvalue associated with $\vv_K$.
The following result follows immediately from the properties of $\vv_K$,
\begin{lemma}\label{lem:smoothness}
Let $\hat \vv$ denote the minimizer of objective Eq. \eqref{eq:objective}. Then $\hat \vv^T L \hat \vv \leq \lambda_K$ .
\end{lemma}
\begin{proof}
Property \RNum{1} and assumption (i) imply $\AUC(\vy_\LL,(\vv_K)_\LL) =1$. Since $J(\hat \vv) \leq J(\vv_K)$ then,
\[
\hat \vv^T L \hat \vv - \gamma \AUC(\vy_\LL,\hat \vv_\LL) \leq \vv_K^T L \vv_K -\gamma = \lambda_K - \gamma
\]
Rearranging the above inequality and inserting $\AUC(\vy_\LL,\hat \vv_\LL) \leq 1$ completes the proof.
\end{proof}
Lemma \ref{lem:smoothness} implies that the optimizer of Eq. \eqref{eq:objective} is a vector that is smooth with respect to the graph. The smoothness is measured by the total variation $\hat \vv^T L \hat \vv$. This is in contrast to Label propagation, where the result for $n \to \infty$ is a spiked vector.

Let $\VK$ denote the subspace spanned by the leading $K$ eigenvectors of the graph Laplacian matrix. 
In the next theorem, we address the following simpler optimization problem.
\begin{align}\label{eq:objective_K}
\hat \vv^{(K)} &= \argmin_{\vv \in \VK;\|\vv\|=1} J(\vv)
\notag \\
&= \argmin_{\vv \in \VK;\|\vv\|=1} \vv^T L\vv - \gamma \frac{4}{L^2} \sum_{\i \in \PP,j \in \N} \sigma(\vv_i-\vv_j).
\end{align}
Clearly, the outcome of Eq. \eqref{eq:objective_K} differs from that of Eq. \eqref{eq:objective}. The difference depends on the spectra of $L$ and the increase in the eigenvalues after $\lambda_K$. We focus on this simpler problem as the guarantees we derive provide important insights into our approach. Specifically, the guarantees show the benefit of having labeled points close to the border between the classes.  
We denote by $\Delta(\vv)$ the average, over all $L^2/4$ pairs of positive and negative labeled samples,
\begin{equation}\label{eq:delta}
\Delta(\vv) = \frac{4}{L^2} \sum_{i \in \PP,j \in \N} (\vv_i - \vv_j) 
= 
\frac{2}{L} \sum_{i \in \PP} \vv_i - \frac{2}{L} \sum_{i \in \N} \vv_j.
\end{equation}
In addition, we define $\kappa$ as a constant that satisfies, for $\vv_K$
\begin{equation}\label{eq:kappa}
\frac{4}{L^2}\sum_{i \in \PP,j \in \N} \sigma( (v_K)_i-(v_K)_j) = 1/2 + \kappa \Delta(\vv_k).
\end{equation}


The following theorem is proven in the supplementary material. 
\begin{theorem}\label{thm:product}
Let $\hat \vv^{(K)}$ be the minimizer of $J(\vv)$ in Eq. \eqref{eq:objective_K}. 
Under assumptions $(1)-(2)$, for the asymptotic case where $n \to \infty$, and $L = \OO(\log n K^2)$, the inner product between $\hat \vv$ and $\vv_K$ satisfies,
\begin{equation}
|\hat \vv^T \vv_K| \geq 4\kappa -4 \frac{\lambda_K}{\gamma \Delta(\vv_K)} - o(L) 
\end{equation}
With probability that goes to $1$. 
\end{theorem}
\noindent
Theorem \ref{thm:product} provides a lower bound on $|\hat \vv^T \vv_K|$ as a function of the following parameters: (i) The distance of the labeled points from the boundary, as captured by the parameter $\kappa$. 
As the positively and negatively labeled points $v_i,v_j$ get closer to the boundary, the value of $\sigma(v_i-v_j)$ converge to their linear approximation $0.5 + 0.25 (v_i-v_j)$, and $\kappa \to 0.25$, and the leading term in the bound converges to $1$. (ii) The ratio of $\lambda_K/\Delta(\vv_K)$. As $\lambda_K \to 0$, the bound improves, as expected. 
(iii) Importantly, the number of required labeled points scales as $K^2$, but depends only logarithmically in the total number of points $n$.

To illustrate this result, we generated a random set of $n=3000$ points according to a uniform distribution over the rectangle $[0,1] \times [0,\beta]$, where $\beta<1$. We set the true labels by the second coordinate such that $y_i = \mathds 1( (\vx_i)_2 \geq \beta/2)$. We repeated the experiment with $\beta=0.5,0.3,0.15$. For $\beta = 0.5$, the label vector is almost perfectly aligned with the third (non-trivial) eigenvector. Setting $\alpha = 0.15$ yields a harder problem where 
the labels are imperfectly correlated with the $7$-th and $8$-th eigenvectors. Fig. \ref{fig:rectangle} shows the AUC score over the full set of points as a function of the number of labeled points for the three values of $\beta$. Achieving, for example, an AUC score of $0.9$ requires only $8$ labels for $\alpha= 0.5$, $14$ points for $\alpha = 0.3$, and roughly $30$ points for $\alpha = 0.15$.

\begin{figure}
    \centering
    \includegraphics[width=0.9\linewidth]{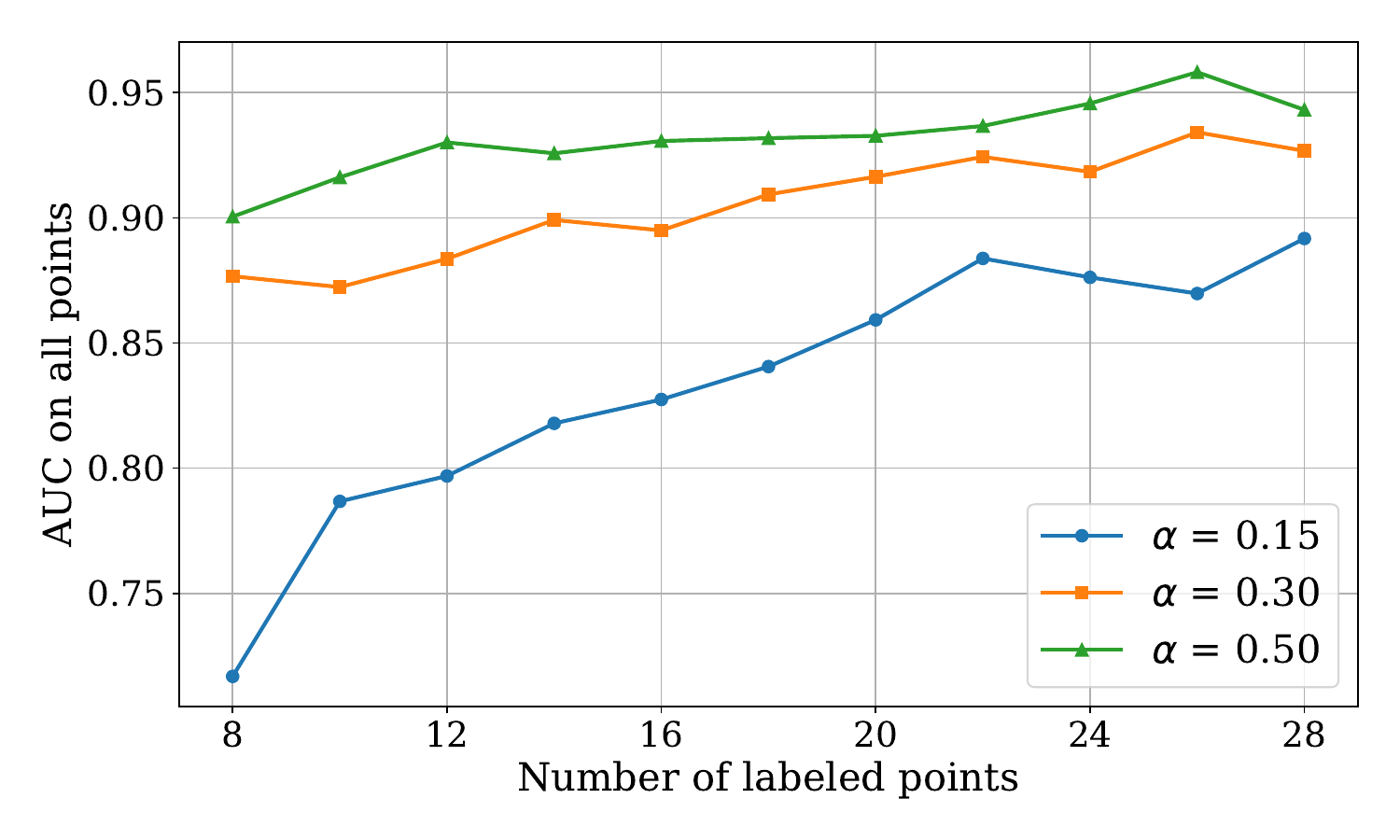}
    \caption{Points sampled over a rectangle of size $[0,1] \times [0,\alpha]$. The labels are set according to the second coordinate. The results show the AUC score, over the unlabeled points, of the vector computed by AUC-spec.}
    \label{fig:rectangle}
\end{figure}

\section{Experimental results}\label{sec:results}

\begin{table*}
\centering
\small
\begin{tabular}{l ccccc ccccc }
\toprule
\makecell{ Methods} & \multicolumn{5}{c}{AUC} & \multicolumn{5}{c}{Accuracy} \\
 & 8 & 10 & 12 & 20 & 200 & 8 & 10 & 12 & 20 & 200 \\
\midrule
AUC-Spec & \textbf{0.797} & 0.780 & \textbf{0.815} & \textbf{0.831} & \textbf{0.896} & \textbf{0.708} & 0.694 & \textbf{0.738} & \textbf{0.747} & \textbf{0.809} \\
Label propagation & 0.789 & 0.767 & 0.810 & 0.817 & 0.877 & 0.506 & 0.510 & 0.504 & 0.501 & 0.602 \\
Top eigenvectors & 0.748 & 0.727 & 0.759 & 0.797 & 0.874 & 0.687 & 0.675 & 0.701 & 0.732 & 0.797 \\
p laplace & 0.762 & 0.781 & 0.799 & 0.829 & 0.893 & 0.700 & 0.708 & 0.721 & 0.745 & 0.806 \\
Poisson learning & 0.770 & \textbf{0.785} & 0.793 & 0.825 & 0.874 & 0.707 & \textbf{0.711} & 0.715 & 0.742 & 0.782 \\
\bottomrule
\end{tabular}
\caption{AUC and accuracy for the Fashion MNIST dataset as a function of the number of labeled points}
\label{tab:fashion_mnist}
\end{table*}

\begin{table*}[htb]
\centering
\small
\begin{tabular}{l ccccc ccccc }
\toprule
\makecell{ Methods} & \multicolumn{5}{c}{AUC} & \multicolumn{5}{c}{Accuracy} \\
 & 8 & 10 & 12 & 20 & 200 & 8 & 10 & 12 & 20 & 200 \\
\midrule
AUC-Spec & \textbf{1.4} & 2.0 & \textbf{1.6} & \textbf{2.0} & 2.8 & \textbf{1.6} & 2.6 & \textbf{1.8} & \textbf{2.0} & \textbf{2.0} \\
Label propagation & 3.2 & 3.2 & 3.0 & 3.8 & 2.8 & 3.8 & 3.4 & 3.8 & 3.8 & 4.2 \\
Top eigenvectors & 4.6 & 4.8 & 5.0 & 4.8 & 4.0 & 4.2 & 4.0 & 4.0 & 3.8 & 3.2 \\
p\_laplace\_learn & 3.8 & 3.4 & 2.8 & 2.4 & \textbf{2.0} & 3.2 & 3.0 & 2.8 & 2.8 & 3.0 \\
Poisson learning & 2.0 & \textbf{1.6} & 2.6 & \textbf{2.0} & 3.4 & 2.2 & \textbf{2.0} & 2.6 & 2.6 & 2.6 \\
\bottomrule
\end{tabular}
\caption{Mean rank of methods over five datasets as a function of the number of labeled points.}
\label{tab:rank}
\end{table*}

\begin{table}[h]
    \centering
    \setlength{\tabcolsep}{3.5pt}
    \vspace{1em}
    \begin{tabular}{lccccc}
        \toprule
        \# labeled points & \textbf{8} & \textbf{10} & \textbf{12} & \textbf{20} & \textbf{200} \\
        \midrule
        AUC-Spec (Ours) & 85.2 & 63.4 & 52.7 & 19.0 & 2.7 \\
        Label Propagation & 89.2 & 91.0 & 89.6 & 90.5 & 38.4 \\
        Top Eigen & 0.01 & 0.01 & 0.02 & 0.02 & 0.03 \\
        P Laplace Learn & 58.3 & 58.6 & 58.5 & 58.7 & 58.9 \\
        Poisson Learn & 132.8 & 134.0 & 133.5 & 132.1 & 134.0 \\
        \bottomrule
    \end{tabular}
    \vspace{1em}    
    \caption{Average run time in seconds over 10 trials.}
    \label{tab:run_time_results}
\end{table}

In this section, we evaluate our method's performance applied to different binary classification tasks on various data sets. 
We compare our method against the following established graph-based semi-supervised learning algorithms: (i) Label Propagation (\cite{zhu2003semi}), (ii) Leading Eigenvectors (\cite{belkin2004semi}), (iii) $p$-Laplace learning (\cite{flores2022analysis}), and (iv) Poisson Learning (\cite{calder2020poisson}).
In all our experiments, we ran 50 trials, randomly choosing which data points are labeled.
As for the methods' hyperparameter, in $p$-Laplace learning we used $p=5$ (in \cite{flores2022analysis} original paper, the experiments were applied with $p=5$ and $p=9$). 
For the leading eigenvectors approach, we used the top $k=5$ vectors, which empirically yielded robust and competitive results. In our method, we applied 20 power iterations with step size $\gamma=0.1$, followed by iterations with $\gamma = 0.01$ until convergence. 
For Label Propagation, we used the implementation available at the scikit-learn Python package. \footnote {\url{https://scikit-learn.org/stable/modules/generated/sklearn.semi_supervised.LabelPropagation.html}} For Poisson Learning and $p$-Laplace, we used the implementation provided by the authors in the GraphLearning Python package.\footnote{\url{https://github.com/jwcalder/GraphLearning}}
For all methods and all datasets, we constructed a weighted graph as in Eq. \eqref{eq:adaptive_scaling}. The adaptive scaling $d_K(\mathbf{x}_i)$ was set according to the $20$-th nearest neighbors. 
We evaluate performance by two criteria measured over the unlabeled points: (i) AUC of the continuous predictive vector with respect to the true labels $\mathbf{y}$ and (ii) the accuracy of the predicted labels. 

\subsection{Description of the datasets}
We tested all methods on five datasets: (i) MNIST \cite{lecun2010mnist}, digits 8,9. (ii) Fashion Mnist \cite{xiao2017fashion}, items "T-shirt/top". (iii) ring-of-Gaussians \cite{holtz2023semi} (iv) Magic Gamma Telescope \cite{magic_gamma_telescope_159} signal vs. background, and (v) CIFAR-10 \cite{krizhevsky2009learning} color images, classes  automobile and cat. We used a publicly available embedded images taken from \citet{calder2020poisson}. The autoencoder architecture applied to compute the embedding is from \cite{zhang2019aet}.

In all datasets, we randomly selected $n=10000$ with an equal proportion for the two classes. The size of the Ring of Gaussians dataset remains $n=3000$. We performed experiments with $8,10,12,20$ and $200$ labeled points, with an equal number of points for the two classes.

\subsection{Discussion on simulation results}
Tables \ref{tab:fashion_mnist} in the main text and tables \ref{tab:magic_dataset},\ref{tab:mnist},\ref{tab:ring_of_gaussians} and \ref{tab:cifar_dataset} in the supplementary material report the AUC and prediction accuracy of all methods across the five benchmark datasets and for various number of labeled points. 
Table \ref{tab:rank} presents the mean rank (1-5, where smaller means higher accuracy) over all datasets. Though the ranking of AUC-spec varies, it ranks either first or close second in the vast majority of experiments.  
The advantage over other methods is particularly significant in small number of labels. 
The AUC ranking of our approach seems to be better than the accuracy ranking, which suggests the results may be improved by a better selection of a threshold over the predictive vector. 

\paragraph{Runtime comparison}
To assess computational efficiency, we measured the average runtime of all methods on a graph with 10,000 nodes constructed from the MNIST dataset. The number of labeled points varied from 8 to 200, as in other experiments. Each method was allowed a maximum of 1,000 iterations with a tolerance of $1e-4$, and runtimes were averaged over 10 independent trials conducted on a standard CPU. The results are in Table \ref{tab:run_time_results}. 
The runtime of our approach reduces with the number of labels, which suggests that our approach converges faster in this regime. The same phenomenon occurs for Label Propagation, though the difference is not as significant as in AUC-spec. The runtimes of P-Laplace and Poisson learning are not affected by the number of labeled points. 

\section{Conclusion and Future work}
In this work, we developed AUC-spec, a method that combines graph representation with the goal of maximizing class discrimination. We believe this approach opens up several promising directions for future research. For example, one could develop new methods by integrating AUC-spec with more robust graph-based techniques, such as p-Laplacian or Poisson learning. Future work may also explore adaptive schemes for jointly tuning the tradeoff between smoothness and discrimination, as well as scalable algorithms for large-scale graphs. Finally, extending the framework to optimize other discrimination-aware metrics—such as partial AUC or fairness-aware objectives—could further enhance its applicability to real-world semi-supervised learning scenarios.


\newpage

\bibliography{aaai2026}
\bibliographystyle{ICML/icml2026}

\newpage
\appendix
\onecolumn

\appendix

\section{Proof of Theorem \ref{thm:product}}.
For the proof of Theorem \ref{thm:product}, we use a concentration result, stated in the following auxiliary lemma.
\begin{lemma}\label{lem:aux_concentration}
    Let $\mathbf r$ denote a set of $L$ indices sampled at random from $1,\ldots,n$ without replacement and let $\Delta_r = \frac{1}{L} \sum_{i \in \mathbf r} \vv_{i}$ denote the average of corresponding elements in $\vv$. Then
    for any $\vv \in \VKM$ we have
    \[
    |\vv_i| \leq \tilde C \frac{\sqrt{K}}{\sqrt{n}}. 
    \]    
\end{lemma}
\begin{proof}
The proof relies on the Hoeffding concentration inequality for the case of sampling without replacement, and a simple $\epsilon$-net argument.
Since the elements of the leading eigenvectors are bounded by $C/\sqrt{n}$, we have that for any vector $\vv \in \VKM$, its elements are bounded by $|v_i|\leq C\frac{\sqrt{K}}{{\sqrt{n}}}$.
For a fixed vector $\vv \in \VKM$, the Hoeffding bound (see for example proposition 1.2 in \cite{bardenet2015concentration}) states that
\[
\Pr\bigg(\big|\frac{1}{L} \sum_{i \in \mathbf r} v_i\big|\geq t\bigg) \leq 2 \exp\bigg(-\frac{-t^2 L}{B^2}\bigg),
\]
where $B = C\frac{\sqrt{K}}{{\sqrt{n}}}$ is the bound on the elements of $\vv$.
Next, consider an $\epsilon$-net on the subspace of $\VKM$. The size of the net is bounded by $(3/\epsilon)^K$ (See for example \cite{vershynin2018high}, corollary 4.2.11). Using the union bound over the net, we have that,
\[
\Pr\bigg(\big|\frac{1}{L} \sum_{i \in \mathbf r} v_i\big|\geq t\bigg) \leq 2 (3/\epsilon)^K \exp\bigg(-\frac{-t^2 L}{B^2}\bigg).
\]
for all vectors $\vv$ in the net. 
Finally, the difference in the average $\frac{1}{L} v_i$ between any vector in $\VKM$ and its closest neighbor in the net is bounded is bounded by $\epsilon/L$.
Thus,
\[
\Pr\bigg(\big|\frac{1}{L} \sum_{i \in \mathbf r} v_i\big|\geq t+\epsilon/L\bigg) \leq 2 (\epsilon/3)^K \exp\bigg(-\frac{-t^2 L}{B^2}\bigg).
\]
We set $\epsilon = 1/\sqrt{n}$ and $t = \frac{\alpha}{\sqrt{L n}}$ For $L =\OO(K^2 \log n)$ we have that there is some constant $C'$
\[
\Pr\bigg(\big|\frac{1}{L} \sum_{i \in \mathbf r} v_i\big|\geq C' \frac{1}{L\sqrt{n}}\bigg)\leq \exp(-C'). 
\]
\end{proof}

Next, we prove the main theorem in four steps.
As mentioned, for simplicity we assume that the number of positive labeled and negative labeled is equal to $L/2$. The loss function is defined as,
\[
\ell(\vv) = \vv^T L \vv -\gamma \frac{4}{L^2}\sum_{i \in \PP,j \in \N} \sigma(v_i-v_j) 
\]


\begin{itemize}
    \item Step 1: For $\vv_k$, by the definition of $\kappa$ in Eq. \eqref{eq:kappa} 
    \begin{equation}\label{eq:upper_bound}
    \ell(\vv) = \vv^T L \vv -\gamma \frac{4}{L^2}\sum_{i \in \PP,j \in N} \sigma( (v_K)_i-(v_K)_j) = \lambda_K - \gamma/2 - \gamma \kappa \Delta(\vv_k),
    \end{equation}
    where recall from Eq. \eqref{eq:delta} that $\Delta(\vv_K)$ is equal to
    \begin{equation}
    \Delta(\vv_k) = \frac{4}{L^2} \sum_{i \in \PP,j \in N} (v_K)_i-(v_K)_j = \frac2L \sum_{i \in \PP} (v_K)_i - \frac2L \sum_{j \in \N} (v_K)_j
    \end{equation}
    \item Step 2: We consider a vector $\vv_\alpha$ equal to
    \[
    \vv_\alpha = \alpha \vv_K + \beta \vv  \qquad \text{such that} \qquad \beta^2+\alpha^2 = 1
    \]
    where $\vv \in \VKM$. 
    We use the concavity fact that $\sigma(x)$ is monotonicly increasing and the inequality $\sigma(x) \leq 0.5 + 0.25x$ for $x \geq 0$ to bound the loss function of $\vv_\alpha$ via
    \begin{align}\label{eq:lower_bound}
    \ell(\vv_\alpha) &\geq \alpha^2 \lambda_K - \gamma \frac{4}{L^2}\sum_{i \in \PP j  \in \N} \sigma(\alpha (v_K)_i +\beta \vv_i - \alpha  (v_K)_j -\beta \vv_j) 
    \notag \\
    & \geq \alpha^2\lambda_K -\gamma \frac{4}{L^2} \sum_{i \in \PP j  \in \N} \Big( \frac12 + 
    \frac14 (\alpha (v_K)_i +\beta \vv_i - \alpha  (v_K)_j -\beta \vv_j) \Big) 
    \notag \\
    &= 
    \alpha^2\lambda_K -\gamma \frac{1}{2} - \gamma \alpha \frac{1}{4}\Delta(\vv_K) - \frac14 \gamma \beta \Delta (\vv)
    \end{align}
    \item step 3: Here we use the properties of the product of manifold. The vector $\vv \in \VKM$  is independent of $\theta_K$ and hence independent of the labels $y$. 
    Thus, we can use the concentration result derived in lemma \ref{lem:aux_concentration}. The lemma proves a bound on the average of $L$ random elements sampled without replacement from \textbf{any} vector $\vv \in \VKM$. Thus, we can assume that under the assumption of the theorem, we have that $\Delta(\vv)$ is bounded by $\frac{C'}{\sqrt{n}L}$ with probability at least $\exp(-C')$. 
    
    \item Step 4: Consider the case where $\Delta(\vv)=0$, and assume that $l(\vv_\alpha)<l(\vv_K)$. By combining the bounds from Eqs. \eqref{eq:upper_bound} and \eqref{eq:lower_bound}
    \[
    \lambda_K - \gamma \kappa \Delta(\vv_K) - \alpha^2 \lambda_K + \gamma \alpha \frac{1}{4} \Delta (\vv_K) \geq 0. 
    \]
    This amounts to a quadratic equation in $\alpha$. For the above expression to be positive,
    a necessary condition is that $\alpha > 4\kappa - 4\lambda_K/(\gamma\Delta(\vv_K)$. Otherwise
    \[
    \lambda_K(1-\alpha^2) -\gamma \Delta(\vv_K)(\kappa - \alpha/4) < \lambda_K -\gamma \Delta(\vv_K)(\kappa - \alpha/4) <0.
    \]
    Finally, adding the perturbation bound derived in step 3, the altered requirement is that
    \[
    \alpha > 4\kappa - 4\frac{\lambda_K}{\gamma \Delta(\vv_K)} - \frac{4C'}{L\sqrt{n}}.
    \]
\end{itemize}

\section{Experiments.}

Here, we present the results on datasets that were not shown in the main file. 

\begin{table*}[ht]
\centering
\small
\begin{tabular}{l ccccc ccccc }
\toprule
\makecell{ Methods} & \multicolumn{5}{c}{AUC} & \multicolumn{5}{c}{Accuracy} \\
 & 8 & 10 & 12 & 20 & 200 & 8 & 10 & 12 & 20 & 200 \\
\midrule
AUC-Spec & 0.631 & 0.628 & 0.630 & 0.652 & 0.774 & \textbf{0.643} & 0.627 & \textbf{0.651} & \textbf{0.678} & 0.762 \\
Label propagation & 0.591 & 0.579 & 0.593 & 0.653 & 0.789 & 0.614 & \textbf{0.636} & 0.641 & 0.649 & 0.681 \\
Top eigenvectors & 0.561 & 0.571 & 0.582 & 0.601 & 0.685 & 0.571 & 0.585 & 0.601 & 0.627 & 0.706 \\
p laplace & 0.616 & 0.635 & 0.634 & 0.659 & \textbf{0.793} & 0.622 & 0.627 & 0.644 & 0.668 & \textbf{0.766} \\
Poisson learning & \textbf{0.635} & \textbf{0.641} & \textbf{0.642} & \textbf{0.667} & 0.779 & 0.607 & 0.609 & 0.607 & 0.627 & 0.724 \\
\bottomrule
\end{tabular}
\caption{AUC and accuracy for the Magic dataset as a function of the number of Labeled points}
\label{tab:magic_dataset}
\end{table*}

\begin{table*}
\centering
\small
\begin{tabular}{l ccccc ccccc }
\toprule
\makecell{ Methods} & \multicolumn{5}{c}{AUC} & \multicolumn{5}{c}{Accuracy} \\
 & 8 & 10 & 12 & 20 & 200 & 8 & 10 & 12 & 20 & 200 \\
\midrule
AUC-Spec & 0.947 & 0.951 & \textbf{0.968} & 0.976 & 0.990 & 0.874 & 0.879 & \textbf{0.908} & 0.919 & 0.928 \\
Label propagation & 0.945 & 0.948 & 0.966 & 0.975 & 0.990 & 0.494 & 0.494 & 0.494 & 0.494 & 0.503 \\
Top eigenvectors & 0.865 & 0.869 & 0.898 & 0.914 & 0.969 & 0.777 & 0.780 & 0.823 & 0.845 & 0.918 \\
p laplace & 0.941 & 0.939 & 0.956 & 0.975 & \textbf{0.992} & 0.790 & 0.815 & 0.844 & 0.856 & 0.915 \\
Poisson learning & \textbf{0.950} & \textbf{0.954} & 0.960 & \textbf{0.980} & 0.991 & \textbf{0.877} & \textbf{0.888} & 0.890 & \textbf{0.927} & \textbf{0.952} \\
\bottomrule
\end{tabular}
\caption{AUC and accuracy for the MNIST dataset as a function of the number of Labeled points}
\label{tab:mnist}
\end{table*}

\begin{table*}
\centering
\small
\begin{tabular}{l ccccc ccccc }
\toprule
\makecell{ Methods} & \multicolumn{5}{c}{AUC} & \multicolumn{5}{c}{Accuracy} \\
 & 8 & 10 & 12 & 20 & 200 & 8 & 10 & 12 & 20 & 200 \\
\midrule
AUC-Spec & \textbf{0.797} & 0.780 & \textbf{0.815} & \textbf{0.831} & \textbf{0.896} & \textbf{0.708} & 0.694 & \textbf{0.738} & \textbf{0.747} & \textbf{0.809} \\
Label propagation & 0.789 & 0.767 & 0.810 & 0.817 & 0.877 & 0.506 & 0.510 & 0.504 & 0.501 & 0.602 \\
Top eigenvectors & 0.748 & 0.727 & 0.759 & 0.797 & 0.874 & 0.687 & 0.675 & 0.701 & 0.732 & 0.797 \\
p laplace & 0.762 & 0.781 & 0.799 & 0.829 & 0.893 & 0.700 & 0.708 & 0.721 & 0.745 & 0.806 \\
Poisson learning & 0.770 & \textbf{0.785} & 0.793 & 0.825 & 0.874 & 0.707 & \textbf{0.711} & 0.715 & 0.742 & 0.782 \\
\bottomrule
\end{tabular}
\caption{AUC and accuracy for the Fashion MNIST dataset as a function of the number of Labeled points}
\label{tab:fashion_mnist}
\end{table*}

\begin{table*}[ht]
\centering
\small
\begin{tabular}{l ccccc ccccc }
\toprule
\makecell{ Methods} & \multicolumn{5}{c}{AUC} & \multicolumn{5}{c}{Accuracy} \\
 & 8 & 10 & 12 & 20 & 200 & 8 & 10 & 12 & 20 & 200 \\
\midrule
AUC-Spec & \textbf{0.842} & \textbf{0.887} & \textbf{0.925} & \textbf{0.982} & 0.995 & 0.677 & 0.731 & 0.800 & 0.928 & 0.986 \\
Label propagation & 0.579 & 0.655 & 0.740 & 0.950 & \textbf{1.000} & \textbf{0.693} & \textbf{0.754} & \textbf{0.825} & \textbf{0.960} & \textbf{0.992} \\
Top eigenvectors & 0.622 & 0.652 & 0.648 & 0.660 & 0.679 & 0.655 & 0.675 & 0.663 & 0.655 & 0.601 \\
p laplace & 0.581 & 0.586 & 0.705 & 0.844 & 0.990 & 0.622 & 0.666 & 0.745 & 0.821 & 0.981 \\
Poisson learning & 0.824 & 0.860 & 0.896 & 0.966 & 0.999 & 0.664 & 0.704 & 0.779 & 0.859 & 0.990 \\
\bottomrule
\end{tabular}
\caption{AUC and accuracy for the Ring of Gaussians dataset as a function of the number of Labeled points}
\label{tab:ring_of_gaussians}
\end{table*}

\begin{table*}[ht]
\centering
\small
\begin{tabular}{l ccccc ccccc }
\toprule
\makecell{ Methods} & \multicolumn{5}{c}{AUC} & \multicolumn{5}{c}{Accuracy} \\
 & 8 & 10 & 12 & 20 & 200 & 8 & 10 & 12 & 20 & 200 \\
\midrule
AUC-Spec & \textbf{0.968} & \textbf{0.981} & 0.981 & 0.984 & 0.998 & 0.887 & 0.887 & 0.907 & 0.908 & 0.934 \\
Label propagation & 0.968 & 0.980 & 0.976 & 0.977 & 0.990 & 0.500 & 0.503 & 0.501 & 0.501 & 0.510 \\
Top eigenvectors & 0.943 & 0.961 & 0.970 & 0.979 & \textbf{0.998} & 0.878 & 0.910 & 0.921 & \textbf{0.940} & \textbf{0.986} \\
p laplace & 0.962 & 0.979 & \textbf{0.982} & \textbf{0.987} & 0.998 & 0.885 & 0.914 & 0.914 & 0.922 & 0.928 \\
Poisson learning & 0.962 & 0.980 & 0.977 & 0.983 & 0.989 & \textbf{0.903} & \textbf{0.927} & \textbf{0.928} & 0.930 & 0.932 \\
\bottomrule
\end{tabular}
\caption{AUC and accuracy for the Cifar dataset as a function of the number of Labeled points}
\label{tab:cifar_dataset}
\end{table*}


\end{document}